\documentclass[dvipdfmx]{article}

\usepackage[utf8]{inputenc} 
\usepackage[T1]{fontenc}    

\usepackage{algorithm}
\usepackage{algorithmicx}
\usepackage{algpseudocode}
\usepackage{amsfonts}
\usepackage{amsmath}
\usepackage{amssymb}
\usepackage{amsthm}
\usepackage{arydshln}
\usepackage{booktabs}
\usepackage{bm}
\usepackage{color}
\usepackage{fancyhdr}
\usepackage[margin = 1.3in]{geometry}
\usepackage{graphicx}
\usepackage{hyperref}
\usepackage{jvlisting}
\usepackage{latexsym}
\usepackage{listings}
\usepackage{mathrsfs}
\usepackage{mathtools}
\usepackage{microtype}
\usepackage{multicol}
\usepackage{multirow}
\usepackage{subcaption}
\usepackage{tcolorbox}
\usepackage{todonotes}  
\usepackage{ulem}
\usepackage{url}
\usepackage{wrapfig}

\usepackage{microtype}      
\usepackage{lipsum}
\usepackage{fancyhdr}       
\usepackage{graphicx}       

\algnewcommand\algorithmicinput{\textbf{Input:}}
\algnewcommand\Input{\item[\algorithmicinput]}
\algnewcommand\algorithmiccontinue{\textbf{continue}}
\algnewcommand\Continue{\item[\algorithmiccontinue]}

\hypersetup{
    setpagesize=false,
    bookmarksnumbered=true,%
    bookmarksopen=true,%
    colorlinks=true,%
    linkcolor=blue,
    citecolor=red,
}

\theoremstyle{definition}
\newtheorem{theorem}{Theorem}[section]
\newtheorem{definition}[theorem]{Definition}

\newtheorem{proposition}[theorem]{Proposition}

\newtheorem{lemma}[theorem]{Lemma}
\newtheorem{assumption}[theorem]{Assumption}

\newtheorem*{theorem*}{Theorem}
\newtheorem*{definition*}{Definition}
\newtheorem*{corollary*}{Corollary}
\newtheorem{proposition*}{Proposition}
\newtheorem*{example*}{Example}
\newtheorem*{exercise*}{Exercise}
\newtheorem{lemma*}[theorem]{Lemma}
\newtheorem{assumption*}[theorem]{Assumption}

\theoremstyle{remark}

\newtheorem{remark*}{Remark}

\numberwithin{equation}{section}

\newcommand{\argmin}{\mathop{\rm arg \, min}\limits}

\makeatletter
\renewenvironment{proof}[1][\proofname]{\par
  \pushQED{\qed}%
  \normalfont \topsep6\p@\@plus6\p@\relax
  \trivlist
  \item\relax
  {\it
  #1\@addpunct{.}}\hspace\labelsep\ignorespaces
}{%
  \popQED\endtrivlist\@endpefalse
}
\makeatother

\newenvironment{acknowledgements}{
    \begin{center}
        \normalfont\textbf{Acknowledgements}  
    \end{center}
    \quotation
}{
    \endquotation
}

\newenvironment{keywords}{
    \begin{center}
        \normalfont\textbf{Keywords}  
    \end{center}
    \quotation
}{
    \endquotation
}

\newenvironment{amsclass}{
    \begin{center}
        \normalfont\textbf{AMS Subject Classifications}  
    \end{center}
    \quotation
}{
    \endquotation
}

\title{\bf

    Designing a Linearized Potential Function \\
    in Neural Network Optimization \\
    Using Csisz\'{a}r Type of Tsallis Entropy

}
\author{
    Keito AKIYAMA
    \\
    Mathematical Institute \\ 
    Tohoku University \\
    6-3, Aramaki Aza-Aoba, Aoba-ku, Sendai 980-8578, Japan\\
    \href{mailto:keito.akiyama.p8@dc.tohoku.ac.jp}{keito.akiyama.p8@dc.tohoku.ac.jp} \\
}
\date{}

\begin{document}

    \maketitle

    \begin{abstract}
        In recent years, 
        learning for neural networks can be viewed 
        as optimization in the space of probability measures. 
        To obtain the exponential convergence to the optimizer, 
        the regularizing term based on Shannon entropy 
        plays an important role. 
        Even though an entropy function heavily affects convergence results, 
        there is almost no result on its generalization, 
        because of the following two technical difficulties: 
        one is the lack of sufficient condition for generalized logarithmic Sobolev inequality, 
        and the other is the distributional dependence of the potential function within the gradient flow equation. 
        In this paper, 
        we establish a framework 
        that utilizes a linearized potential function
        via Csisz\'{a}r type of Tsallis entropy, 
        which is one of the generalized entropies. 
        We also show that
        our new framework enable us to derive an exponential convergence result.  
    \end{abstract}


    \begin{amsclass}
        35Q49, 
        49J20, 
        82C32. 
    \end{amsclass}


    \clearpage

    \section{Introduction}\label{sec:intro}

    \subsection{Problem Backgrounds and Settings}

    In machine learning, 
    which is applied to a wide range of fields 
    such as image recognition~\cite{KSH2012}, 
    natural language processing~\cite{V+2017} 
    and object detection~\cite{R2015}, 
    the objective is to extract some laws about unknown input-output data based on given data. 
    Neural networks are one of the most important models in machine learning. 
    \begin{definition}[A Neural Network]
        Let $\sigma: \mathbb{R} \to \mathbb{R}$ be a nonlinear and measurable function. 
        For $x = (x_0, x') \in \mathbb{R} \times \mathbb{R}^{d - 1}$, 
        \begin{align}
            h_x(z) 
            := x_0 \sigma (x' \cdot z), 
            \quad z \in \mathbb{R}^{d - 1}, 
            \notag
        \end{align}
        is a neural network. 
        The function $\sigma$ is referred to as an activator. 
    
    \end{definition}

    Neural network optimization is to detect suitable $x$
    that fits unknown input-output data, 
    and it is difficult to derive existence and convergence results 
    of a global minimizer since such problems become nonlinear and non-convex. 
    In the context of optimizing a neural network with one hidden layer, 
    the optimization problem with respect to $x$ is reduced to 
    the one with respect to the probability measure $\mu$ which generates $x$, 
    utilizing a mean field theory. 
    Mei et al.~\cite{MMN2018} established the framework of mean field neural networks 
    as an optimization problem of some target functional of $\mu$ adding some regularization term, 
    and showed the existence of an unique optimizer. 
    Hu et al.~\cite{H+2021} derived the first variation of the target functional, 
    and based on this work, 
    Chizat~\cite{Chi2022} showed the exponential convergence result 
    of the solution to the gradient flow equation. 

    We summarize Chizat's pioneering work. 
    Let $\lambda, \tau > 0$ be regularizing parameters. 
    The target functional in \cite{Chi2022} is 
    \begin{align}
        \widetilde{\mathcal{F}}_{\lambda, \tau}(\mu) 
        := \widetilde{\mathcal{G}}_{\lambda}(\mu) 
        + \tau \int_{\mathbb{R}^d} u \ln{u} ~dx, 
        \notag
    \end{align}
    where $\mu$ is the probability measure of parameters $x$, 
    and $u$ is the Radon--Nikodym derivative of $\mu$ with respect to $d$-dimensional Lebesgue measure~$\mathcal{L}^d$. 
    The second term indicates Shannon entropy~\cite{Sh1948} regularization term. 
    The functional $\widetilde{\mathcal{G}}_{\lambda}$ is 
    the sum of the generalization error term and the quadratic moment of $\mu$, 
    that is, 
    \begin{align} 
        \label{eq:priChi}
        \widetilde{\mathcal{G}}_{\lambda}(\mu) 
        := \mathbb{E}_{(Z, Y) \sim \mathcal{D}} 
        \left[ \ell \left( \int_{\mathbb{R}^d} h_x(Z) ~d\mu(x), Y \right) \right] 
        + \dfrac{\lambda}{2} \int_{\mathbb{R}^d} |x|^2 ~d\mu(x). 
    \end{align}
    The rigorous definitions of some notations are given in Section~\ref{subseq:note}. 
    The Wasserstein gradient flow of $\widetilde{\mathcal{F}}_{\lambda, \tau}$ is 
    the following Fokker--Planck equation, 
    whose potential function depends on the solution: 
    \begin{align}
        \partial_t \mu_t 
        = \nabla \cdot (\mu_t \nabla \widetilde{V}_{\lambda}[\mu_t]) 
        + \tau\triangle \mu_t, 
        \notag
    \end{align}
    where $\mu_t = \mu_t(x)$, 
    and $\widetilde{V}_{\lambda}[\mu] \in C^1(\mathbb{R}^d)$ is the Wasserstein first variation, that is, 
    \begin{align}
        \widetilde{V}_{\lambda}[\mu](x) 
        := \mathbb{E}_{(Z, Y) \sim \mathcal{D}} 
        \left[ \partial_1 \ell \left( \int_{\mathbb{R}^d} h_{x'}(Z) ~d\mu(x'), Y \right) h_x(z) \right] 
        + \dfrac{\lambda}{2} |x|^2. 
        \notag
    \end{align}
    Based on these settings, 
    Chizat derived an exponential decay property, 
    $\widetilde{\mathcal{F}}_{\lambda, \tau}(\mu_t) 
    - \widetilde{\mathcal{F}}_{\lambda, \tau}(\mu^*) = O(e^{-Ct})$,  
    for the distributional solution $(\mu_t)_{t \geq 0}$ of the flow, 
    some minimizer $\mu^*$ of $\widetilde{\mathcal{F}}_{\lambda, \tau}$, 
    and some constant $C > 0$. 
    The key steps of the proof in \cite{Chi2022} are as follows. 
    Due to the properties of the Wasserstein gradient flow, 
    the time derivative of target functional at the distributional solution 
    is equal to the negative 
    of relative Fisher information with respect to $e^{- \frac{\widetilde{V}_{\lambda}[\mu]}{\tau}} \mathcal{L}^d$.
    The logarithmic Sobolev inequality, 
    which indicates that Shannon relative entropy is less than or equal to Fisher Information,  
    derives the differential inequality of the target functional: 
    \begin{align}
        \dfrac{d}{dt} (\widetilde{\mathcal{F}}_{\lambda, \tau}(\mu_t) 
        - \widetilde{\mathcal{F}}_{\lambda, \tau}(\mu^*)) 
        \leq - \int_{\mathbb{R}^d} (\tau u_t \ln{u_t} + u_t \widetilde{V}_{\lambda}[\mu_t]) ~dx, 
        \notag
    \end{align}
    where the right hand side indicates Shannon relative entropy 
    between $\mu_t = u_t\mathcal{L}^d$ and $e^{- \frac{\widetilde{V}_{\lambda}[\mu]}{\tau}} \mathcal{L}^d$. 
    It provides the exponential decay property thanks to Gronwall inequality. 
    The crucial parts of the proof 
    are two criteria for a probability measure to satisfy the logarithmic Sobolev inequality: 
    Bakry--\'{E}mery criterion~\cite{BE1985} 
    and Holley--Stroock criterion~\cite{HS1987}. 
    \begin{itemize}
        \item 
        Bakry--\'{E}mery criterion: 
        If $V_0 \in C^1(\mathbb{R}^d)$ satisfies 
        $\nabla^2 V_0 \succeq \rho I_d$ for some $\rho > 0$, 
        then $e^{- V_0}\mathcal{L}^d$ satisfies the logarithmic Sobolev inequality. 

        \item 
        Holley--Stroock criterion: 
        If $\psi \in L^{\infty}$ 
        and a probability measure $\nu$ satisfies the logarithmic Sobolev inequality, 
        then $e^{-\psi}\nu$ also satisfies the logarithmic Sobolev inequality. 
    \end{itemize}
    If an activator of a neural network is bounded 
    and additional assumptions are provided, 
    then $e^{- \frac{\widetilde{V}_{\lambda}[\mu]}{\tau}}\mathcal{L}^d$ satisfies the logarithmic Sobolev inequality.
    Chizat applied the two criteria, 
    with $V_0$ as the quadratic regularization term, 
    and $\psi$ as the generalization error term. 

    In this paper, we try to generalize the Chizat's framework, 
    so that more wide class of entropies can be considered as a regularizing term. 
    To this end, we focus on Tsallis entropy,  
    a generalization of Shannon entropy, 
    which is defined by using some convex function.  
    In 1988, 
    Tsallis developed a new statistical mechanics~(Tsallis statistical mechanics) 
    to represent phenomena
    with power-type function for the steady state, 
    such as chaos and fractal~\cite{T1988}. 
    There are two types of Tsallis relative entropy 
    between $\mu = u \mathcal{L}^d$ and $\nu = v \mathcal{L}^d$ . 
    One is Bregman type~\cite{N2004}
    \begin{align}
        \int_{\mathbb{R}^d} \left\{ (F(u) - uF'(v)) - (F(v) - vF'(v)) \right\} ~dx
        \notag
    \end{align}
    for some convex function $F : \mathbb{R} \to \mathbb{R}$, 
    and another is Csisz\'{a}r type~\cite{Cs1967, CS2004}
    \begin{align}
        \int_{\mathbb{R}^d} v \varphi \left( \dfrac{u}{v} \right) ~dx
        \notag
    \end{align}
    for some convex function $\varphi : \mathbb{R} \to \mathbb{R}$. 
    Both types are equivalent up to constant when $F(s) = \varphi(s) = s\ln{s} - s$, 
    and correspond to Shannon relative entropy.  
    There are several references concerning Tsallis entropy. 
    For a generalized Sobolev inequality, refer to \cite{CGH2003}. 
    For a comprehensive geometric analysis of Tsallis entropy within the framework of information geometry, 
    see \cite{Am2016}. 
    In recent works, 
    for stability results related to the uncertainty principle 
    in the context of Tsallis entropy, see \cite{Su2025}.

    Our aim is to analyze some effects of Tsallis entropy 
    in the optimization problem for neural networks. 
    To this end, 
    we have to overcome the following two difficulties. 
    One is the restriction of Holley--Stroock criterion for generalized Sobolev inequality.  
    Even though Holley--Stroock criterion partially holds 
    for Csisz\'{a}r type of Tsallis entropy, 
    which is known as the stability for a perturbation of $\varphi$-Sobolev inequality~\cite{Cha2004},  
    but it is unknown for Bregman type of Tsallis entropy. 
    The other difficulty is the distributional dependence of the potential function. 
    It generates the reminder term of the time derivative of the potential, 
    which prevents us from analyzing the functional. 
    Some detailed technical difficulties and observations 
    concerning Bregman type of Tsallis entropy 
    or dependence of the potential function are given 
    in our another paper~\cite{Ak2024}. 
    
    The main result of this paper is 
    designing a framework for neural network optimization 
    via Csisz\'{a}r type of Tsallis entropy 
    by using some linearized potential. 
    In our framework, 
    the primary term $\mathcal{G}_{\lambda}$ is defined as 
    \begin{align}
        \label{eq:priAk}
        \mathcal{G}_{\lambda}(\mu)
        := 
        \mathbb{E}_{(Z, Y) \sim \mathcal{D}} \left[ 
            \int_{\mathbb{R}^d} 
            \ell (h_{x}(Z), Y) 
            ~d\mu(x)
        \right] 
        + \dfrac{\lambda}{2} \int_{\mathbb{R}^d} |x|^2 ~d\mu(x).  
    \end{align}
    Comparing Chizat's \eqref{eq:priChi} to our \eqref{eq:priAk}, 
    \eqref{eq:priAk} means 
    the average of the loss of neural networks,  
    whereas \eqref{eq:priChi} indicates  
    the loss of the average of neural networks. 
    This exchange makes the primary term $\mathcal{G}_{\lambda}$ linear. 
    In addition that, 
    the potential function $V_{\lambda} \in C^1(\mathbb{R}^d)$ becomes independent on $\mu$: 
    \begin{align} 
        \label{eq:linpot}
        V_{\lambda}(x)
        = \mathbb{E}_{(Z, Y) \sim \mathcal{D}} [\ell (h_{x}(Z), Y)] 
        + \dfrac{\lambda}{2} |x|^2. 
    \end{align}
    We set the target functional $\mathcal{E}_{\varphi, \lambda, \tau}(\mu) $ as follows. 
    \begin{definition}[Target functional]
        For a probability measure $\mu$ which is absolutely continuous 
        with respect to Lebesgue measure, 
        \begin{align} \label{eq:funalA}
            \mathcal{E}_{\varphi, \lambda, \tau}(\mu) 
            := 
                \int_{\mathbb{R}^d} 
                    \varphi \left(u e^{\frac{V_{\lambda}}{\tau}} \right) e^{- \frac{V_{\lambda}}{\tau}}~dx, 
        \end{align}
        where
        $u$ is the Radon--Nikodym derivative of $\mu$ with respect to $\mathcal{L}^d$~($\mu = u \mathcal{L}^d$), 
        and $\varphi : [0, \infty) \to [0, \infty)$ is a function 
        satisfying Assumption~\ref{ass:ass} below. 
    \end{definition}
    We can derive the exponential decay property of $\mathcal{E}_{\varphi, \lambda, \tau}$ 
    by employing the linear Kolmogorov--Fokker--Planck equation, 
    which is widely summarized in \cite{B+2015}
    and quickly surveyed in \cite{DL2018} , 
    More precise statement 
    of our main result will be stated in Theorem~\ref{thm:main} 
    after preparing some notations in Section~\ref{subseq:note}. 

    Here is an additional remark for our framework. 
    By replacing $\varphi$ with $\varphi_{1, \tau}(s) := \tau (s \ln{s} - (s - 1))$, 
    $\mathcal{E}_{\varphi, \lambda, \tau}(\mu)$ is equivalent up to a constant to 
    the functional added by Shannon entropy regularization term: 
    \begin{align}
        \mathcal{E}_{\varphi_{1, \tau}, \lambda, \tau}(\mu) 
        = \tau \int_{\mathbb{R}^d} u \ln{u} ~dx + \int_{\mathbb{R}^d} u V_{\lambda} dx. 
        \notag
    \end{align}
    It could be said that 
    the functional $\mathcal{E}_{\varphi, \lambda, \tau}$ contains 
    the optimization problem for neural networks via Shannon entropy using the linearized potential. 

    Our contributions are summarized as follows: 
    \begin{itemize}
        \item 
        We established a framework for an optimization problem 
        of one-hidden-layered neural network via Csisz\'{a}r type of Tsallis entropy, 
        utilizing the linearized potential function~\eqref{eq:linpot}. 
        Our formulation contains the cases utilizing Shannon entropy regularization term. 

        \item 
        We obtained the existence and the exponential decay results 
        in the context of our framework. 
    \end{itemize}
    We provide a figure comparing the position of recent works, 
    the work presented in this paper, 
    and the works that are remained open. 
    Some technical difficulties appeared in open cases are reported in \cite{Ak2024}. 
    \begin{table}[H]
        \centering
        \begin{tabular}{c||c:c|c:c}
            \multicolumn{1}{r||}{Entropy} & \multicolumn{2}{c|}{Shannon} & \multicolumn{2}{c}{Tsallis} \\ \cline{2-5}
            \multicolumn{1}{l||}{Potential} & Csisz\'{a}r & Bregman & Csisz\'{a}r & Bregman \\ \hline
            dependent on $\mu$ & open & \cite{Chi2022} & open & open \\ \hline
            independent on $\mu$ & \multicolumn{2}{c|}{\it This Paper} & {\it This Paper} & open \\
        \end{tabular}
    \end{table}


    The rest of this paper is organized as follows:
    In Section~\ref{subseq:note}, 
    we list the notations of some symbols. 
    In Section~\ref{sec:notate}, 
    we prepare some analytic assumptions and tools, 
    and state the main theorem in Theorem~\ref{thm:main}. 
    In Section~\ref{sec:pre}, 
    we introduce some auxiliary lemmas for proof of the main theorem. 
    In Section~\ref{sec:proof}, 
    we prove Theorem~\ref{thm:main}. 
    In Section~\ref{sec:appen}, 
    we state a detailed proof of the existence of the solution 
    of Kolmogorov--Fokker--Planck equation. 

    \subsection{Notations} \label{subseq:note}

    Let $\mathcal{L}^d$ be $d$-dimensional Lebesgue measure. 
    For a set $\mathcal{Y} \subset \mathbb{R}$, 
    we call a continuous function
    $\ell: \mathcal{Y} \times \mathcal{Y} \ni (y_1, y_2) \mapsto \ell(y_1, y_2) \in [0, \infty)$ 
    as a loss function, 
    where $\ell(y_1, y_2)$ represents an error between $y_1$ and $y_2$. 
    For example, 
    a squared loss function: $\ell(y_1, y_2) := \dfrac{1}{2}(y_1 - y_2)^2$, 
    and a logistic loss function: $\ell(y_1, y_2) := \ln{(1 + \exp{(-y_1 y_2)})}$. 
    

    For sets $\mathcal{Z} \subset \mathbb{R}^d$ and
    $\mathcal{Y} \subset \mathbb{R}$,  
    a measure $\mathcal{D}$ on $\mathcal{Z} \times \mathcal{Y}$, 
    and a $\mathcal{D}$-measurable function $f$, 
    we denote a mean of $f$ with respect to $\mathcal{D}$ by
    \begin{align}
        \mathbb{E}_{(Z, Y) \sim \mathcal{D}} [f(Z, Y)]
        := \iint_{\mathcal{Z} \times \mathcal{Y}} f(z, y) ~d\mathcal{D}(z, y),  
        \notag
    \end{align}
    where $(Z, Y)$ indicates a $\mathcal{Z} \times \mathcal{Y}$-valued random variable. 
    Under these notations, 
    \begin{align}
        \mathbb{E}_{(Z, Y) \sim \mathcal{D}} \left[ \ell (h_{x}(Z), Y) \right]
        \notag
    \end{align}
    is called by a generalization error of the neural network $h_x$.

    \paragraph{Function Spaces.}

    We define $C_b$ 
    as the space of continuous and bounded real functions on $\mathbb{R}^d$, 
    and $C_0^{\infty}$ as the space of smooth functions 
    with compact support in $\mathbb{R}^d$. 
    For a measure $\nu$ on $\mathbb{R}^d$
    ~(which is not necessarily $\mathcal{L}^d$), 
    $L^p(\nu) \equiv L^p(\mathbb{R}^d; \nu)$~($1 \leq p \leq \infty$) denotes 
    the space of all $\nu$-measurable 
    and $p$-integrable functions on $\mathbb{R}^d$. 
    Let $W^{1, p}(\nu) \equiv W^{1, p}(\mathbb{R}^d; \nu)$ 
    be the Sobolev space on $\mathbb{R}^d$ with respect to the measure $\nu$, 
    and in particular, 
    $H^1(\nu) := W^{1, 2}(\nu)$. 
    In case $\nu = \mathcal{L}^d$, 
    we just note $L^p$, $W^{1, p}$ and $H^1$, respectively. 
    By $C^1([0, \infty); H^1(\nu))$ 
    we denote the space of all $C^1$ functions 
    from $[0, \infty)$ to $H^1(\nu)$, 
    which is called as Bochner space. 

    $\mathcal{P} \equiv \mathcal{P}(\mathbb{R}^d)$ is 
    the set of Borel probability measures on $\mathbb{R}^d$. 
    $\mathcal{P}_2$ denotes the set of all probability measures with finite quadratic moments,  
    and $\mathcal{P}_2^{\rm ac}$ is the subspaces of $\mathcal{P}_2$, 
    which consists of absolutely continuous measures with respect to $\mathcal{L}^d$. 
    In other words,  
    \begin{align}
        \mathcal{P}_2 \equiv \mathcal{P}_2(\mathbb{R}^d) 
        &:= \left\{ \mu \in \mathcal{P}: \int_{\mathbb{R}^d} |x|^2 ~d\mu(x) < \infty \right\}, 
        \notag \\
        \mathcal{P}_2^{\rm ac} \equiv \mathcal{P}_2^{\rm ac}(\mathbb{R}^d) 
        &:= \left\{ \mu \in \mathcal{P}_2: \mu \ll \mathcal{L}^d \right\}. 
        \notag
    \end{align}
    For $\mu_1, \mu_2 \in \mathcal{P}_2$, 
    the following $W_2$ defines a distance in $\mathcal{P}_2$, 
    and $(\mathcal{P}_2, W_2)$ is a complete and separable metric space: 
    \begin{align}
        W_2(\mu_1, \nu_2) 
        := \min_{\bm{\mu} \in \Gamma(\mu_1, \mu_2)} 
        \left\{ \int_{\mathbb{R}^d \times \mathbb{R}^d} |x_1 - x_2|^2 ~d\bm{\mu}(x_1, x_2) \right\}^{\frac{1}{2}}, 
        \notag
    \end{align}
    where $\Gamma(\mu_1, \mu_2)$ indicates the class of transport plans between $\mu_1$ and $\mu_2$. 
    For some detailed property of $(\mathcal{P}_2, W_2)$, 
    please refer to \cite{AGS2008, AS2007, Sa2015}.  


    



    \section{Preliminaries and Main Theorem}\label{sec:notate}


    We impose some additional conditions 
    for a loss function $\ell$, 
    an activation function $\sigma$,  
    and an convex function $\varphi$ for Csisz\'{a}r type of Tsallis entropy. 

    \begin{assumption}[Assumptions for $\ell$, $\sigma$ and $\varphi$] \; \label{ass:ass}
        Our assumptions on $\ell$, $\sigma$ and $\varphi$ are the following:
        \begin{enumerate}
            \item[(L)] 
            $\ell \in C^2(\mathcal{Y} \times \mathcal{Y})$ is bounded. 

            \item[(S)] 
            $\sigma \in C^2(\mathbb{R})$. 


            \item[(P1)] 
            $\varphi \in C^2([0, \infty))$. 
            
            \item[(P2)] 
            $\varphi(1) = 0$. 

            \item[(P3)] 
            $\varphi$ is non-negative and strictly convex. 


            \item[(P4)] 
            $\varphi$ is superlinear, 
            that is, 
            $\displaystyle\lim_{s \to \infty} \dfrac{\varphi(s)}{s} = \infty$. 
        \end{enumerate}    
    \end{assumption}

    A typical example of functions satisfying (S) 
    is $\sigma(s) := \dfrac{1}{2}(1 + \arctan(s))$. 
    Some example functions satisfying (P1) to (P4) includes
    $\varphi_{1, \tau}(s) := \tau (s \ln{s} - (s - 1))$
    and $\varphi_{q, \tau}(s) 
    := \dfrac{\tau}{q - 1}(s^q - 1 - q(s - 1))$ for $q > 1$. 



    Deriving an convergence to a minimizer of $\mathcal{E}_{\varphi, \lambda, \tau}(\mu)$, 
    we consider the following Kolmogorov--Fokker--Planck equation: 
    \begin{align} \label{eq:KFP}
        \begin{cases}
            \partial_t w 
            = \triangle w - \tau^{-1} \nabla V_{\lambda} \cdot \nabla w, & (t, x) \in (0, \infty) \times \mathbb{R}^d, \\
            w(0, \cdot) = w_0, & x \in \mathbb{R}^d. 
        \end{cases}
    \end{align}

    \begin{definition}[$C^1([0, \infty); H^1(\gamma))$-solution of \eqref{eq:KFP}] 
        Let $\gamma := e^{-\frac{V_{\lambda}}{\tau}} \mathcal{L}^d$. 
        $w \in C^1([0, \infty); H^1(\gamma))$ is 
        a solution of \eqref{eq:KFP}, 
        if the identities
        \begin{align}
            \begin{dcases}
                \int_{\mathbb{R}^d} \dfrac{\partial w}{\partial t}(t, x)\phi(x) ~d\gamma(x) 
                = - \int_{\mathbb{R}^d} \nabla w(t, x) \cdot \nabla \phi(x) ~d\gamma(x), 
                & (t, x) \in (0, \infty) \times \mathbb{R}^d, \\
                w(0, \cdot) = w_0, & x \in \mathbb{R}^d 
            \end{dcases}
            \notag
        \end{align}
        hold for all $\phi \in C^{\infty}_0(\gamma)$. 
    \end{definition}

    \begin{proposition} \label{prop:HY}
        If $w_0 \in H^1(\gamma) \cap L^{\infty}(\gamma)$ 
        holds that $w_0^- \equiv 0$~($\gamma$-a.e.) 
        and $\|w_0\|_{L^1(\gamma)} = 1$,  
        then, 
        there exists an unique $C^1([0, \infty); H^1(\gamma))$-solution of \eqref{eq:KFP} $w$. 
        Moreover, 
        the following properties are satisfied: 
        for all $t \in [0, \infty)$, 
        \begin{enumerate}
            \item[(a)]
            $w(t) \in L^{\infty}(\gamma)$. 

            \item[(b)] 
            ${w(t)}^- \equiv 0$~($\gamma$-a.e.).
            
            \item[(c)] 
            $\|w(t)\|_{L^1(\gamma)} = 1$. 
        \end{enumerate}
    \end{proposition}

    \begin{proof}[Outlined proof of Proposition~\ref{prop:HY}]
        The operator 
        $L_{\lambda} 
        := \triangle - \tau^{-1}\nabla V_{\lambda} \cdot \nabla$
        is a maximal monotone operator on $H^1(\gamma)$. 
        It follows 
        that there exists an unique $C^1([0, \infty); H^1(\gamma))$-solution of \eqref{eq:KFP} 
        according to Hille--Yosida theorem. 
        By a direct calculation, 
        the properties (a), (b) and (c) are also followed. 
        A detailed proof is provided in the appendix.
    \end{proof}

    Under these notations and assumptions, 
    we present our main theorem below. 

    \begin{theorem}[Main Theorem] \label{thm:main}
        Let $\mathcal{Z} \subset \mathbb{R}^d$ and
        $\mathcal{Y} \subset \mathbb{R}$ be bounded sets,  
        and $\mathcal{D}$ be a finite measure on $\mathcal{Z} \times \mathcal{Y}$. 
        Provide that 
        $\ell$, $\sigma$, and $\varphi$ fulfills Assumption~\ref{ass:ass}. 
        For $\mu \in \mathcal{P}_2$, 
        we define $\mathcal{E}_{\varphi, \lambda, \tau}(\mu)$ as \eqref{eq:funalA}. 
        Then, 
        there exists an unique minimizer $\mu_{\varphi, \lambda, \tau}^*$ 
        of $\mathcal{E}_{\varphi, \lambda, \tau}$, 
        that is, $\mu_{\varphi, \lambda, \tau}^* $ satisfies
        \begin{align}
            \mu_{\varphi, \lambda, \tau}^* 
            = \argmin_{\mu \in \mathcal{P}_2} \mathcal{E}_{\varphi, \lambda, \tau}(\mu). 
            \notag
        \end{align}
    
        Let $\gamma := e^{-\frac{V_{\lambda}}{\tau}} \mathcal{L}^d$. 
        For arbitrary $w_0 \in H^1(\gamma) \cap L^{\infty}(\gamma)$ 
        with
        $w_0^- \equiv 0$~($\gamma$-a.e.) and 
        $\|w_0\|_{L^1(\gamma)} = 1$, 
        consider an unique solution $w \in C^1([0, \infty); H^1(\gamma))$ of \eqref{eq:KFP}. 
        Define $\mu_0, \mu_t \in \mathcal{P}_2^{\rm ac}$ 
        as $\mu_0 = w_0 \gamma$ 
        and $\mu_t = w(t, \cdot) \gamma$, respectively. 
        If $\mathcal{E}_{\varphi, \lambda, \tau}(\mu_0) < \infty$, 
        then
        \begin{align}
            \mathcal{E}_{\varphi, \lambda, \tau}(\mu_t) 
            - \mathcal{E}_{\varphi, \lambda, \tau}(\mu_{\varphi, \lambda, \tau}^*) 
            \leq e^{- \Lambda t} \left( \mathcal{E}_{\varphi, \lambda, \tau}(\mu_0) 
            - \mathcal{E}_{\varphi, \lambda, \tau}(\mu_{\varphi, \lambda, \tau}^*) \right), 
            \notag
        \end{align}
        where $\Lambda := 2\lambda \tau^{-1} e^{- 2M \tau^{-1}}$ is a constant, 
        and $M > 0$ satisfies the inequality \eqref{eq:bdd}. 
    \end{theorem}

    \section{Some Auxiliary Lemmas} \label{sec:pre}

    In this section, 
    we state some lemmas for the proof of Theorem~\ref{thm:main}.



    \begin{lemma}[Boundedness of the generalization error term] \label{lem:bdd}
        There exists $M > 0$ 
        such that the inequality
        \begin{align} \label{eq:bdd}
            |\mathbb{E}_{(Z, Y) \sim \mathcal{D}} [\ell (h_{x}(Z), Y)]| 
            \leq M
        \end{align}
        holds for all $x \in \mathbb{R}^d$. 
    \end{lemma}

    \begin{proof}
        The claim of lemma follows 
        from the boundedness of $\mathcal{Z}$ and $\mathcal{Y}$, 
        the finiteness of $\mathcal{D}$, 
        and the continuity and the boundedness of $\ell$. 
    \end{proof}

    \begin{lemma}[Finiteness of $\gamma$]
        A measure $\gamma := e^{-\frac{V_{\lambda}}{\tau}} \mathcal{L}^d$ 
        is finite. 
    \end{lemma}

    \begin{proof}
        Employing Gaussian integral,
        \begin{align}
            \int_{\mathbb{R}^d} ~d\gamma(x) 
            &= \int_{\mathbb{R}^d} 
            e^{- \frac{1}{\tau}\mathbb{E}_{(Z, Y) \sim \mathcal{D}} [\ell (h_{x}(Z), Y)]}
            e^{- \frac{\lambda}{2\tau} |x|^2} ~dx 
            \notag 
            \\
            &\leq e^{\frac{M}{\tau}} \int_{\mathbb{R}^d} 
            e^{- \frac{\lambda}{2\tau} |x|^2} ~dx 
            \notag 
            \\
            &= e^{\frac{M}{\tau}} \left( \dfrac{2\pi \tau}{\lambda} \right)^{\frac{d}{2}}
            \notag 
            \\ 
            &< \infty, 
            \notag
        \end{align}
        which indicates the boundedness of $\gamma$. 
    \end{proof}

    \begin{lemma}[Decomposition property of convex function]\label{lem:lindecomp}
        Under the assumption (P1), (P3) and (P4), 
        there exists a strictly increasing function $\psi: [0, \infty) \to \mathbb{R}$, 
        such that
        \begin{align}
            \lim_{s \to \infty} \psi(s) 
            = \infty, 
            \notag
        \end{align}
        and
        \begin{align}
            \varphi(s) = s \psi(s) + \varphi(0).
            \notag
        \end{align}
    \end{lemma}

    \begin{proof}[Proof]
        We define $\psi$ as
        \begin{align}
            \psi(s) 
            := 
            \begin{dcases}
                \dfrac{\varphi(s) - \varphi(0)}{s} & s \in (0, \infty), \\
                \varphi'(s) & s = 0. 
            \end{dcases}
            \notag
        \end{align}
        This $\psi$ is strictly increasing. 
        Indeed, 
        for all $0 < s < t < \infty$, 
        \begin{align}
            \dfrac{\varphi(s) - \varphi(0)}{s} 
            < \dfrac{\varphi(t) - \varphi(0)}{t}
            \notag
        \end{align}
        holds since $\varphi$ is strictly convex. 
        As $s \to 0$, 
        we found that $\psi$ is strictly increasing on $[0, \infty)$.
        It follows that $\displaystyle\lim_{s \to \infty} \psi(s) = \infty$ 
        due to the superlinearity of $\varphi$. 
        By direct calculations, 
        the equality
        \begin{align}
            \varphi(s) 
            = s \dfrac{\varphi(s) - \varphi(0)}{s} + \varphi(0)
            \equiv s \psi(s) + \varphi(0)
            \notag
        \end{align}
        holds.
    \end{proof}




    The following four lemmas are stated only in terms of their claims.

    \begin{lemma}[Dunford--Pettis Theorem~{\cite[Theorem~4.30.]{B2011}}]
        For arbitrary bounded sequence 
        $(f_n)_{n \in \mathbb{N}} \subset L^1$, 
        the following two properties are equivalent:
        \begin{enumerate}
            \item[(i)] 
            $(f_n)_{n \in \mathbb{N}}$ is uniformly integrable, 
            that is, 
            for all $\varepsilon > 0$, 
            there exists $\delta > 0$ such that for any measurable set $E$,
            if $\gamma(E) < \delta$, 
            then $\int_E |f_n| d\gamma < \varepsilon$ for all $n \in \mathbb{N}$.  
            Or the same thing, 
            \begin{align}
                \lim_{K \to \infty} \sup_{n \in \mathbb{N}} \int_{|f_n| \geq K} |f_n| ~d\gamma = 0. 
                \notag
            \end{align}

            \item[(ii)]
            Every subsequence of $(f_n)_{n \in \mathbb{N}}$ has a subsequence 
            that converges weakly in $L^1$. 
        \end{enumerate}
    \end{lemma}


    \begin{lemma}[Duality formula for $\mathcal{E}_{\varphi, \lambda, \tau}$~{\cite[Lemma~9.4.4.]{AGS2008}}]
        For any $\mu \in \mathcal{P}$, 
        we have 
        \begin{align}
            \mathcal{E}_{\varphi, \lambda, \tau}(\mu) 
            &= \sup_{S^* \in C_b} \left\{ 
                \int_{\mathbb{R}^d} S^*(x) ~d\mu(x) 
                - \int_{\mathbb{R}^d} \varphi^*(S^*(x)) ~d\gamma(x)
            \right\}, 
            \notag
        \end{align}
        where $\varphi^*: [0, \infty) \to [0, \infty)$ is Legendre conjugate function of $\varphi$, 
        that is, 
        \begin{align}
            \varphi^* (s^*) 
            := \sup_{s \geq 0} \left\{ s \cdot s^* - \varphi(s) \right\}. 
            \notag
        \end{align}
    \end{lemma}

    \begin{lemma}[Bakry--\'{E}mery criterion
        ~\cite{BE1985},{\cite[Corollary~9.]{Cha2004}}]
        \label{lem:BE}
        If $V_0 \in C^1(\mathbb{R}^d)$ satisfies 
        $\nabla^2 V_0 \geq \rho$ for some $\rho > 0$, 
        then $e^{- V_0}\mathcal{L}^d$ satisfies the $\varphi$-Sobolev inequality
        with a constant $\dfrac{1}{2\rho}$, 
        that is, 
        \begin{align}
            \int_{\mathbb{R}^d} \varphi(\zeta) e^{- V_0} ~dx 
            \leq \dfrac{1}{2\rho} 
            \int_{\mathbb{R}^d} \varphi''(\zeta) |\nabla \zeta|^2 e^{- V_0} ~dx
            \notag
        \end{align}
        for all $\zeta \in H^1(\gamma)$. 
    \end{lemma}

    \begin{lemma}[Holley--Stroock criterion
        ~\cite{HS1987},{\cite[Proposition~17.]{Cha2004}}] 
        \label{lem:HS}
        If $\psi \in L^{\infty}$ 
        and a probability measure $\nu$ satisfies the $\varphi$-Sobolev inequality, 
        with a constant $\rho_0 > 0$, 
        then $e^{-\psi}\nu$ also satisfies the $\varphi$-Sobolev inequality
        with a constant $e^{2\|\psi\|_{L^{\infty}}} \rho_0$, 
        that is, 
        \begin{align}
            \int_{\mathbb{R}^d} \varphi(\zeta) e^{-\psi} ~d\nu(x) 
            \leq e^{- 2\|\psi\|_{L^{\infty}}} \dfrac{1}{2 \rho_0} 
            \int_{\mathbb{R}^d} \varphi''(\zeta) |\nabla \zeta|^2 e^{-\psi} ~d\nu(x). 
            \notag
        \end{align}
    \end{lemma}

    Based on Lemma~\ref{lem:BE} and Lemma~\ref{lem:HS}, 
    we prove the $\varphi$-Sobolev inequality for the target functional $\mathcal{E}_{\varphi, \lambda, \tau}$. 

    \begin{lemma}[$\varphi$-Sobolev inequality for $\mathcal{E}_{\varphi, \lambda, \tau}$] \label{lem:PS}
        For all $\mu \in \mathcal{P}_2^{\mathrm{ac}}$, 
        \begin{align}
            \mathcal{E}_{\varphi, \lambda, \tau}(\mu) 
            = \int_{\mathbb{R}^d} \varphi(w) e^{- \frac{V_{\lambda}}{\tau}} ~dx 
            \leq e^{\frac{2M}{\tau}} \dfrac{\tau}{2\lambda} \int_{\mathbb{R}^d} \varphi''(w) |\nabla w|^2 
            e^{- \frac{V_{\lambda}}{\tau}} ~dx,   
            \notag
        \end{align}
        where $\mu = we^{- \frac{V_{\lambda}}{\tau}} \mathcal{L}^d$. 
    \end{lemma}

    \begin{proof}
        Let $\nu_2 := e^{- \frac{\lambda}{2\tau} |x|^2} \mathcal{L}^d$. 
        Since the inequality
        \begin{align}
            \nabla^2 {e^{- \frac{\lambda}{2\tau} |x|^2}} 
            \succeq \dfrac{\lambda}{\tau} I_d
            \notag
        \end{align}
        holds, 
        $\nu_2$ satisfies $\varphi$-Sobolev functional inequalities 
        with a constant $\dfrac{\tau}{2\lambda} > 0$ 
        due to Lemma~\ref{lem:BE}. 

        The remaining calculations 
        are analogous to the method used in 
        \cite[Proposition~17]{Cha2004}. 
        By the boundedness of the generalization term~(Lemma~\ref{lem:bdd}), 
        \begin{align}
            \mathcal{E}_{\varphi, \lambda, \tau}(\mu)
            &= \int_{\mathbb{R}^d} \varphi(w) e^{- \frac{V_{\lambda}}{\tau}} ~dx 
            \notag 
            \\
            &= \int_{\mathbb{R}^d} \varphi(w) 
            e^{- \frac{1}{\tau}\mathbb{E}_{(Z, Y) \sim \mathcal{D}} [\ell (h_{x}(Z), Y)]}
            e^{- \frac{\lambda}{2\tau} |x|^2} ~dx
            \notag 
            \\
            &\leq 
            \int_{\mathbb{R}^d} \varphi(w) e^{\frac{M}{\tau}} ~d\nu_2(x)
            \notag 
            \\
            &\leq e^{\frac{M}{\tau}} \dfrac{\tau}{2\lambda}
            \int_{\mathbb{R}^d} \varphi''(w) |\nabla w|^2 ~d\nu_2(x) 
            \notag 
            \\
            &= e^{\frac{M}{\tau}} \dfrac{\tau}{2\lambda}
            \int_{\mathbb{R}^d} \varphi''(w) |\nabla w|^2 
            e^{\frac{1}{\tau}\mathbb{E}_{(Z, Y) \sim \mathcal{D}} [\ell (h_{x}(Z), Y)]} 
            e^{- \frac{V_{\lambda}}{\tau}} ~dx 
            \notag 
            \\
            &\leq e^{\frac{2M}{\tau}} \dfrac{\tau}{2\lambda}
            \int_{\mathbb{R}^d} \varphi''(w) |\nabla w|^2 
            e^{- \frac{V_{\lambda}}{\tau}} ~dx,   
            \notag 
        \end{align}
        Thus, we obtain the desired estimate. 
    \end{proof}

    \section{Proof of the Main Theorem} \label{sec:proof}

    \subsection{Existence and Uniqueness of Minimizer}

    The proof is based on a direct method in the calculus of variations, 
    and which is analogous to the approach used in 
    \cite[Proposition~4.1]{JKO1998}. 
    
    By the positivity of $\varphi$ (which is assumed in (P3)), 
    $\mathcal{E}_{\varphi, \lambda, \tau} \geq 0$, 
    which indicates $\mathcal{E}_{\varphi, \lambda, \tau}$ is bounded below. 
    Let $(\mu_n)_{n \in \mathbb{N}} \subset \mathcal{P}_2^{\rm ac}$ be 
    a minimizing sequence of $\mathcal{E}_{\varphi, \lambda, \tau}$, 
    that is, $(\mu_n)_{n \in \mathbb{N}}$ satisfies
    \begin{align}
        \lim_{n \to \infty} \mathcal{E}_{\varphi, \lambda, \tau}(\mu_n) 
        = \inf_{\mu \in \mathcal{P}_2^{\rm ac}}\mathcal{E}_{\varphi, \lambda, \tau}(\mu). \;
        (\geq 0)
        \notag
    \end{align}
    Since $\mu_n \in \mathcal{P}_2^{\rm ac}$ for each $n \in \mathbb{N}$, 
    we can denote 
    $\mu_n = u_n \mathcal{L}^d$, 
    where $u_n \in L^1$ is the non-negative Radon--Nikodym derivative of $\mu_n$. 

    \paragraph{The lower semicontinuity of $\mathcal{E}_{\varphi, \lambda, \tau}$.}

    Fix a convergent sequence $(\mu_n)_{n \in \mathbb{N}} \subset \mathcal{P}_2^{\rm ac}$ arbitrary. 
    There exists a $\mu \in \mathcal{P}_2^{\rm ac}$, 
    such that $\displaystyle\lim_{n \to \infty} W_2(\mu_n, \mu) = 0$. 
    Applying a duality formula, 
    \begin{align}
        \liminf_{n \to \infty} \mathcal{E}_{\varphi, \lambda, \tau}(\mu_n) 
        &= \liminf_{n \to \infty}
        \left\{ 
            \sup_{S^* \in C_b} 
            \left( 
                \int_{\mathbb{R}^d} S^*(x) d\mu_n(x) 
                - \int_{\mathbb{R}^d} \varphi^*(S^*(x)) d\gamma(x) 
            \right)
        \right\}
        \notag 
        \\
        &\geq \liminf_{n \to \infty}
        \left( 
            \int_{\mathbb{R}^d} T^*(x) d\mu_n(x) 
            - \int_{\mathbb{R}^d} \varphi^*(T^*(x)) d\gamma(x) 
        \right)
        \notag 
        \\
        &= 
        \int_{\mathbb{R}^d} T^*(x) d\mu(x) 
        - \int_{\mathbb{R}^d} \varphi^*(T^*(x)) d\gamma(x), 
        \notag 
    \end{align}
    for all $T^* \in C_b$. 
    The last equality is due to the equivalence 
    between $W_2$ convergence and narrowly convergence
    ~({\cite[Proposition~7.1.5]{AGS2008}}). 
    By the arbitrariness of $T^*$, 
    \begin{align}
        \liminf_{n \to \infty} \mathcal{E}_{\varphi, \lambda, \tau}(\mu_n) 
        \geq \sup_{T^* \in C_b} 
        \left( 
            \int_{\mathbb{R}^d} T^*(x) d\mu(x) 
            - \int_{\mathbb{R}^d} \varphi^*(T^*(x)) d\gamma(x) 
        \right) 
        = \mathcal{E}_{\varphi, \lambda, \tau}(\mu). 
        \notag 
    \end{align}
    Thus, 
    $\mathcal{E}_{\varphi, \lambda, \tau}$ is lower semicontinuous. 

    \paragraph{Existence of a convergent subsequence of $(\mu_n)_{n \in \mathbb{N}}$. }


    The sequence $(u_n)_{n \in \mathbb{N}} \subset L^1$ is bounded, 
    because $\| u_n \|_{L^1} = 1$ for all $n \in \mathbb{N}$. 

    Fix $K > 0$ and $n \in \mathbb{N}$ arbitrary. 
    $(\mathcal{E}_{\varphi, \lambda, \tau}(\mu_n))_{n \in \mathbb{N}}$ is convergent, 
    which also implies that it is bounded, 
    that is, 
    \begin{align} \label{eq:unibdd}
        \sup_{n \in \mathbb{N}} \int_{\mathbb{R}^d} \varphi \left(u_n e^{\frac{V_{\lambda}}{\tau}} \right) e^{- \frac{V_{\lambda}}{\tau}}~dx 
        < \infty. 
    \end{align}

    By Lemma~\ref{lem:lindecomp}, 
    \begin{align}
        \label{eq:funaldecomp}
        \int_{\mathbb{R}^d} \varphi \left(u_n e^{\frac{V_{\lambda}}{\tau}} \right) e^{- \frac{V_{\lambda}}{\tau}}~dx 
        = \int_{\mathbb{R}^d} u_n \psi\left(u_n e^{\frac{V_{\lambda}}{\tau}} \right) ~dx 
        + \varphi(0) \int_{\mathbb{R}^d} d\gamma(x)
    \end{align}
    holds. 
    Combining with \eqref{eq:unibdd} and \eqref{eq:funaldecomp}, 
    \begin{align}
        \sup_{n \in \mathbb{N}} \int_{\mathbb{R}^d} u_n \psi \left(u_n e^{\frac{V_{\lambda}}{\tau}} \right) ~dx 
        < \infty. 
        \notag
    \end{align}
    As $\psi$ is monotone increasing, 
    we can calculate as
    \begin{align}
        \int_{\{u_n \geq K\}} u_n ~dx 
        &\leq \int_{\{u_n \geq K\}} u_n 
        \psi \left(u_n e^{\frac{V_{\lambda}}{\tau}} \right) 
        \psi \left(u_n e^{\frac{\inf V_{\lambda}}{\tau}} \right)^{-1} ~dx
        \notag 
        \\
        &\leq \psi \left(K e^{\frac{\inf V_{\lambda}}{\tau}} \right)^{-1} 
        \int_{\{u_n \geq K\}} u_n \psi \left(u_n e^{\frac{V_{\lambda}}{\tau}} \right) ~dx 
        \notag
        \\
        &\leq \psi \left(K e^{\frac{\inf V_{\lambda}}{\tau}} \right)^{-1}
        \sup_{n \in \mathbb{N}} \int_{\mathbb{R}^d} u_n \psi \left(u_n e^{\frac{V_{\lambda}}{\tau}} \right) ~dx, 
        \notag
    \end{align}
    which indicates
    \begin{align}
        \sup_{n \in \mathbb{N}} \int_{\{u_n \geq K\}} u_n ~dx 
        \leq \psi \left(K e^{\frac{\inf V_{\lambda}}{\tau}} \right)^{-1} 
        \sup_{n \in \mathbb{N}} \int_{\mathbb{R}^d} u_n \psi \left(u_n e^{\frac{V_{\lambda}}{\tau}} \right) ~dx. 
        \notag
    \end{align}
    Since 
    \begin{align}
        \lim_{K \to \infty} \psi \left(K e^{\frac{\inf V_{\lambda}}{\tau}} \right)^{-1} \to 0, 
        \notag
    \end{align}
    due to the monotone increasing property of $\psi$, 
    we can obtain the uniform integrability of $(u_n)_{n \in \mathbb{N}}$, that is, 
    \begin{align}
        \lim_{K \to \infty} \sup_{n \in \mathbb{N}} \int_{\{u_n \geq K\}} u_n ~dx 
        = 0. 
        \notag
    \end{align}

    Consequently, 
    there exists a weakly convergent subsequence 
    $(u_{n_k})_{k \in \mathbb{N}} \subset (u_n)_{n \in \mathbb{N}}$ in $L^1$, 
    by Dunford--Pettis theorem. 

    \paragraph{Existence of minimizer.}

    Let $u^*$ be the $L^1$-weak limit of $(u_{n_k})_{k \in \mathbb{N}}$, 
    and set $\mu_{\varphi, \lambda, \tau}^* = u^* \mathcal{L}^d$. 
    By the lower semicontinuity of $\mathcal{E}_{\varphi, \lambda, \tau}$, 
    \begin{align}
        \inf_{\mu \in \mathcal{P}_2^{\rm ac}}\mathcal{E}_{\varphi, \lambda, \tau}(\mu) 
        \leq \mathcal{E}_{\varphi, \lambda, \tau}(\mu_{\varphi, \lambda, \tau}^*)
        \leq \liminf_{k \to \infty} \mathcal{E}_{\varphi, \lambda, \tau}(\mu_{n_k})
        = \inf_{\mu \in \mathcal{P}_2^{\rm ac}}\mathcal{E}_{\varphi, \lambda, \tau}(\mu), 
        \notag
    \end{align}
    which indicates
    \begin{align}
        \mathcal{E}_{\varphi, \lambda, \tau}(\mu_{\varphi, \lambda, \tau}^*) 
        = \inf_{\mu \in \mathcal{P}_2^{\rm ac}}\mathcal{E}_{\varphi, \lambda, \tau}(\mu). 
        \notag
    \end{align}
    Thus, 
    $\mu_{\varphi, \lambda, \tau}^* 
    = \argmin_{\mu \in \mathcal{P}_2^{\rm ac}} \mathcal{E}_{\varphi, \lambda, \tau}(\mu)$. 

    \paragraph{Uniqueness of minimizer.}

    It is enough to show the strict convexity of $\mathcal{E}_{\varphi, \lambda, \tau}$. 
    Indeed, 
    for all $\mu_1, \mu_2 \in \mathcal{P}_2^{\rm ac}$, 
    let $u_1$ and $u_2$ denote the Radon--Nikodym derivatives, respectively. 
    Since $\varphi$ is strictly convex, 
    For all $t \in [0, 1]$, 
    \begin{align}
        \mathcal{E}_{\varphi, \lambda, \tau}((1 - t)\mu_1 + t \mu_2) 
        &= \int_{\mathbb{R}^d} \varphi \left(((1 - t)u_1 + tu_2) e^{\frac{V_{\lambda}}{\tau}} \right) e^{- \frac{V_{\lambda}}{\tau}}~dx 
        \notag 
        \\
        &= \int_{\mathbb{R}^d} \varphi \left( (1 - t)u_1 e^{\frac{V_{\lambda}}{\tau}} + tu_2 e^{\frac{V_{\lambda}}{\tau}} \right) e^{- \frac{V_{\lambda}}{\tau}}~dx 
        \notag 
        \\
        &< \int_{\mathbb{R}^d} \left\{ (1 - t)\varphi \left(u_1 e^{\frac{V_{\lambda}}{\tau}} \right) + t\varphi \left(u_2 e^{\frac{V_{\lambda}}{\tau}} \right) \right\} e^{- \frac{V_{\lambda}}{\tau}}~dx 
        \notag 
        \\
        &= (1 - t)\int_{\mathbb{R}^d} \varphi \left(u_1 e^{\frac{V_{\lambda}}{\tau}} \right) e^{- \frac{V_{\lambda}}{\tau}}~dx 
        + t\int_{\mathbb{R}^d} \varphi \left(u_2 e^{\frac{V_{\lambda}}{\tau}} \right) e^{- \frac{V_{\lambda}}{\tau}}~dx 
        \notag 
        \\
        &= (1 - t)\mathcal{E}_{\varphi, \lambda, \tau}(\mu_1) 
        + t\mathcal{E}_{\varphi, \lambda, \tau}(\mu_2).
        \notag
    \end{align}
    Thus, 
    the minimizer of $\mathcal{E}_{\varphi, \lambda, \tau}$ is unique. 


    \subsection{Exponential Decay Property of $\mathcal{E}_{\varphi, \lambda, \tau}$}



    By Lebesgue's dominant convergence theorem, 
    \begin{align}
        \dfrac{d}{dt} \mathcal{E}_{\varphi, \lambda, \tau}(\mu_t) 
        &= \int_{\mathbb{R}^d} 
        \partial_t \left( \varphi(w) e^{- \frac{V_{\lambda}}{\tau}} \right) ~dx
        \notag 
    \end{align}
    holds. 
    Since 
    $\varphi'(w) \in H^1(\gamma)$
    due to 
    $\varphi \in C^2$
    and $w \in L^{\infty}(\gamma)$, 

    \begin{align}
        \dfrac{d}{dt} \mathcal{E}_{\varphi, \lambda, \tau}(\mu_t) 
        &= \int_{\mathbb{R}^d} 
        \partial_t \left( \varphi(w) e^{- \frac{V_{\lambda}}{\tau}} \right) ~dx
        \notag 
        \\
        &= \int_{\mathbb{R}^d} \varphi'(w) \partial_t w 
        e^{- \frac{V_{\lambda}}{\tau}} ~dx 
        \notag
        \\
        &= - \int_{\mathbb{R}^d} \nabla (\varphi'(w)) 
        \cdot \nabla w
        e^{- \frac{V_{\lambda}}{\tau}} ~dx 
        \notag
        \\
        &= - \int_{\mathbb{R}^d} \varphi''(w) |\nabla w|^2 
        e^{- \frac{V_{\lambda}}{\tau}} ~dx.  
        \notag 
    \end{align}
    Combining Lemma~\ref{lem:PS} and the above calculation, 
    \begin{align}
        \dfrac{d}{dt} (\mathcal{E}_{\varphi, \lambda, \tau}(\mu_t) 
        - \mathcal{E}_{\varphi, \lambda, \tau}(\mu_{\varphi, \lambda, \tau}^*))
        &\leq - e^{-\frac{2M}{\tau}} \dfrac{2\lambda}{\tau} \mathcal{E}_{\varphi, \lambda, \tau}(\mu_t) 
        \notag 
        \\
        &\leq - e^{-\frac{2M}{\tau}} \dfrac{2\lambda}{\tau} (\mathcal{E}_{\varphi, \lambda, \tau}(\mu_t) 
        - \mathcal{E}_{\varphi, \lambda, \tau}(\mu_{\varphi, \lambda, \tau}^*))
        \notag 
    \end{align}
    due to the positivity of $\mathcal{E}_{\varphi, \lambda, \tau}$. 
    Hence, letting $\Lambda := 2\lambda \tau^{-1} e^{- 2M\tau^{-1}} > 0$,
    we obtain 
    \begin{align}
        \mathcal{E}_{\varphi, \lambda, \tau}(\mu_t) 
        - \mathcal{E}_{\varphi, \lambda, \tau}(\mu_{\varphi, \lambda, \tau}^*)
        \leq e^{- \Lambda t} (\mathcal{E}_{\varphi, \lambda, \tau}(\mu_0) 
        - \mathcal{E}_{\varphi, \lambda, \tau}(\mu_{\varphi, \lambda, \tau}^*)) 
        \notag
    \end{align}
    by Gronwall inequality. 
    This completes the proof of Theorem~\ref{thm:main}.

    \clearpage

    \renewcommand{\thesection}{\Alph{section}}
    \setcounter{section}{0}

    \section{Appendix}\label{sec:appen}



    \subsection{Solvability of Kolmogorov--Fokker--Planck Equation}


    \begin{lemma}[Hille--Yosida Theorem~{\cite[Theorem~7.4.]{B2011}}]
        Let $A$ be a maximal monotone operator on some Hilbelt space $H$. 
        Then, given any $w_0 \in D(A)$,  
        there exists a unique function 
        $w \in C^1([0, \infty); H) \cap C([0, \infty); D(A))$
        satisfying
        \begin{align}
            \begin{cases}
                \dfrac{dw}{dt} + A w = 0, & t \in [0, \infty), \\
                w(0) = w_0.  
            \end{cases}
            \notag
        \end{align}
    \end{lemma}

    \paragraph{Definition of the operator}

    Let $\gamma := e^{-\frac{V_{\lambda}}{\tau}} \mathcal{L}^d$. 
    we define the operator 
    $\widetilde{L} : D(\widetilde{L}) \subset H^1(\gamma) \to H^{-1}(\gamma)$, 
    \begin{align}
        D(\widetilde{L}) 
        &:= H^1(\gamma), 
        \notag 
        \\
        _{H^{-1}(\gamma)}\langle \widetilde{L} w, v \rangle_{H^1(\gamma)}
        &:= \int_{\mathbb{R}^d} \nabla w \cdot \nabla v ~d\gamma(x), 
        \notag
    \end{align}
    for $w, v \in H^1(\gamma)$. 
    By the coercivity and the continuity of the bilinear form, 
    for all $\widetilde{L} w \in H^{-1}(\gamma)$, 
    there exists an unique $\varphi_w \in H^1(\gamma)$, 
    such that
    \begin{align}
        _{H^{-1}(\gamma)}\langle \widetilde{L} w, v \rangle_{H^1(\gamma)}
        = \langle \varphi_w, v \rangle_{H^1(\gamma)},  
        \notag
    \end{align}
    thanks to Lax--Milgram theorem.  
    We can identify $\widetilde{L} w \in H^{-1}(\gamma)$ with $\varphi_w \in H^1(\gamma)$, 
    and define $L : D(L) \subset H^1(\gamma) \to H^1(\gamma)$ as
    \begin{align}
        D(L) 
        &:= H^1(\gamma), 
        \notag 
        \\
        L w
        &:= \varphi_w, 
        \notag
    \end{align}
    for $w \in H^1(\gamma)$.

    \paragraph{Compactness of $L$}

    The operator $L$ is bounded and linear. 
    Indeed, 
    for all $w\in H^1(\gamma)$, 
    by H\"{o}lder's inequality, 
    \begin{align}
        \left\|L w \right\|_{H^1(\gamma)} 
        &= \sup_{\|v\|_{H^1(\gamma)} \leq 1} |\langle \varphi_w, v \rangle_{H^1(\gamma)}| 
        \notag 
        \\
        &= \sup_{\|v\|_{H^1(\gamma)} \leq 1} |_{H^{-1}(\gamma)}\langle \widetilde{L} w, v \rangle_{H^1(\gamma)}| 
        \notag 
        \\
        &\leq \sup_{\|v\|_{H^1(\gamma)} \leq 1} \int_{\mathbb{R}^d} |\nabla w| |\nabla v| ~d\gamma(x) 
        \notag 
        \\
        &\leq \sup_{\|v\|_{H^1(\gamma)} \leq 1} \left\| \nabla w \right\|_{L^2(\gamma)} \left\| \nabla v \right\|_{L^2(\gamma)} 
        \notag 
        \\
        &\leq \sup_{\|v\|_{H^1(\gamma)} \leq 1} \left\| w \right\|_{H^1(\gamma)} \left\| v \right\|_{H^1(\gamma)}. 
        \notag 
        \\
        &= \left\| w \right\|_{H^1(\gamma)}. 
        \notag
    \end{align}
    Moreover, 
    for arbitrary $(w_n)_{n \in \mathbb{N}} \subset B_1$, 
    where $B_1 := \left\{ w \in H^1(\gamma): \|w\|_{H^1(\gamma)} \leq 1 \right\}$, 
    the boundedness of $L$ implies that 
    the sequence $(L w_n)_{n \in \mathbb{N}} \subset \overline{L(B_1)}$ is also bounded. 
    By direct calculations, 
    \begin{align}
        \|\widetilde{L} w_n \|_{H^{-1}(\gamma)} 
        &= \sup_{\|v\|_{H^1(\gamma)} \leq 1} |_{H^{-1}(\gamma)}\langle \widetilde{L} w_n, v \rangle_{H^1(\gamma)}| 
        \notag 
        \\
        &= \sup_{\|v\|_{H^1(\gamma)} \leq 1} | \langle \varphi_{w_n}, v \rangle_{H^1(\gamma)}| 
        \notag 
        \\
        &\leq \|L w_n\|_{H^1(\gamma)} 
        \notag 
        \\ 
        &\leq \|w_n\|_{H^1(\gamma)} 
        \notag 
        \\ 
        &\leq 1,  
        \notag
    \end{align}
    hence $(\widetilde{L} w_n)_{n \in \mathbb{N}} \subset H^{-1}(\gamma)$ is also bounded.
    Due to Ascoli--Arzel\`{a} theorem,
    there exists a subsequence $(\widetilde{L} w_{n_k})_{k \in \mathbb{N}} 
    \subset (\widetilde{L} w_n)_{n \in \mathbb{N}}$
    and the limit $\psi \in H^{-1}(\gamma)$, 
    \begin{align}
        \lim_{k \to \infty} \| \widetilde{L} w_{n_k} - \psi \|_{H^{-1}(\gamma)} 
        = 0
        \notag
    \end{align}
    holds. 
    Thanks to Lax--Milgram theorem, 
    there exists an unique $\varphi_{w_{n_k}}$ and $\varphi \in H^1(\gamma)$, 
    respectively, 
    such that
    \begin{align}
        _{H^{-1}(\gamma)}\langle \widetilde{L} w_{n_k}, v \rangle_{H^1(\gamma)}
        = \langle \varphi_{w_{n_k}}, v \rangle_{H^1(\gamma)}, 
        \quad 
        _{H^{-1}(\gamma)}\langle \psi, v \rangle_{H^1(\gamma)}
        = \langle \varphi, v \rangle_{H^1(\gamma)}. 
        \notag
    \end{align}
    Hence, 
    $(L w_{n_k})_{k \in \mathbb{N}}$ is convergent 
    by the calculation
    \begin{align}
        \limsup_{k \to \infty} \| L w_{n_k} - \varphi \|_{H^1(\gamma)} 
        &\leq \limsup_{k \to \infty} \sup_{\|v\|_{H^1(\gamma)} \leq 1}
        | \langle \varphi_{w_{n_k}} \varphi, v \rangle_{H^1(\gamma)} | 
        \notag 
        \\
        &= \limsup_{k \to \infty} \sup_{\|v\|_{H^1(\gamma)} \leq 1}
        | _{H^{-1}(\gamma)}\langle \widetilde{L} w_{n_k} - \psi, v \rangle_{H^1(\gamma)} |
        \notag 
        \\
        &= \limsup_{k \to \infty} \|\widetilde{L} w_{n_k} - \psi \|_{H^{-1}(\gamma)} 
        \notag 
        \\
        &= 0. 
        \notag
    \end{align}
    Consequently, $L$ is a compact operator 
    because of the compactness of $\overline{L(B_1)}$.  

    \paragraph{Maximal monotonicity of $L$}

    For all $w \in H^1(\gamma)$, 
    \begin{align}
        \langle L w, w \rangle_{H^1(\gamma)} 
        = _{H^{-1}(\gamma)}\langle \widetilde{L} w, w \rangle_{H^1(\gamma)} 
        = \int_{\mathbb{R}^d} |\nabla w|^2 ~d\gamma(x) 
        \geq 0
        \notag 
    \end{align}
    holds, 
    and which indicates the monotonicity of $L$.

    We show the maximality, 
    that is, the equality $R(I - L) = H^1(\gamma)$. 
    For the compact operator $L$, 
    the injectivity and the surjectivity of the operator are equivalent due to Fredholm alternative. 
    It is sufficient to show that ${\rm Ker}(I - L) = \{0\}$. 
    For arbitrary $w \in {\rm Ker}(I - L)$ 
    and $v \in H^1(\gamma)$, 
    \begin{align}
        0 
        &= \langle (I - L)w, v \rangle_{H^1(\gamma)} 
        \notag 
        \\
        &= \langle w, v \rangle_{H^1(\gamma)} 
        + _{H^{-1}(\gamma)}\langle \widetilde{L} w, v \rangle_{H^1(\gamma)} 
        \notag 
        \\
        &= \langle w, v \rangle_{H^1(\gamma)} 
        + \int_{\mathbb{R}^d} \nabla w \cdot \nabla v ~d\gamma(x) 
        \notag 
    \end{align}
    holds. 
    Setting $w = v$, 
    \begin{align}
        0 = \| v \|_{H^1(\gamma)}^2 + \| \nabla v \|_{L^2(\gamma)}^2, 
        \notag
    \end{align}
    which indicates $v = 0$. 
    Hence, 
    ${\rm Ker}(I - L) = \{0\}$. 

    \paragraph{Existence of the $C^1([0, \infty); H^1(\gamma))$-solution of \eqref{eq:KFP}}

    By the maximal monotonicity of $L$, 
    We can apply Hille--Yosida theorem 
    by setting $A = L$, 
    and $H = D(A) = H^1(\gamma)$. 
    There exists $w \in C^1([0, \infty); H^1(\gamma))$ uniquely, 
    such that
    \begin{align}
        \begin{cases}
            \dfrac{dw}{dt} + L w = 0, & t \in [0, \infty), \\
            w(0) = w_0.  
        \end{cases}
        \notag
    \end{align}

    \paragraph{Check for properties (a), (b) and (c)}

    Concerning (a) and (c), 
    Let $(P_t)_{t \geq 0}$ be a $C_0$-contraction semigroup induced by $L$. 
    By the property of the $C_0$-contraction semigroup, 
    we obtain
    $\| w(t) \|_{L^{\infty}(\gamma)} 
    = \| P_t w_0 \|_{L^{\infty}(\gamma)} 
    \leq \|w_0\|_{L^{\infty}(\gamma)}$ 
    and $\|w(t)\|_{L^1(\gamma)} 
    = \|P_t w_0\|_{L^1(\gamma)} 
    = \|w_0\|_{L^1(\gamma)} 
    = 1$. 
    For some detailed property of $(P_t)_{t \geq 0}$, 
    please refer to \cite[Section~6.3.]{AS2007}. 

    Concerning (b), 
    since
    \begin{align}
        0 
        &= \left\langle \dfrac{d}{dt} w + L w, w^- \right\rangle_{L^2(\gamma)} 
        \notag 
        \\
        &= \left\langle \dfrac{d}{dt} w^+ + L w^+, w^- \right\rangle_{L^2(\gamma)} 
        - \left\langle \dfrac{d}{dt} w^- + L w^-, w^- \right\rangle_{L^2(\gamma)}, 
        \notag
    \end{align}
    then 
    \begin{align}
        \left\langle \dfrac{d}{dt} w^- + L w^-, w^- \right\rangle_{L^2(\gamma)}
        = \left\langle \dfrac{d}{dt} w^+ + L w^+, w^- \right\rangle_{L^2(\gamma)}
        = 0. 
        \notag
    \end{align}
    On the other hand, 
    \begin{align}
        \left\langle \dfrac{d}{dt} w^- + L w^-, w^- \right\rangle_{L^2(\gamma)} 
        &= \left\langle \dfrac{d}{dt} w^-, w^- \right\rangle_{L^2(\gamma)} 
        + \left\langle L w^-, w^- \right\rangle_{L^2(\gamma)}. 
        \notag
    \end{align}
    By the monotone propety of $L$, 
    the second term is positive, which indicates that 
    \begin{align}
        \dfrac{1}{2} \dfrac{d}{dt} \|w^-\|_{L^2(\gamma)}^2 
        = \left\langle \dfrac{d}{dt} w^-, w^- \right\rangle_{L^2(\gamma)} 
        \leq 0. 
        \notag
    \end{align}
    Consequently, 
    \begin{align}
        \|w^-\|_{L^2(\gamma)}^2  
        \leq \|w_0^-\|_{L^2(\gamma)}^2 
        = 0, 
        \notag
    \end{align}
    which indicates that $w(t)^- \equiv 0$, $\gamma$-a.e.

    \begin{acknowledgements}
        This work was partially funded 
        by JSPS KAKENHI (grant number 21K18582)
        and JST SPRING (grant number JPMJSP2114), 
        and supported in part 
        by the WISE Program for AI Electronics in Tohoku University.
        The author would like to thank Professor Norisuke Ioku
        for his helpful suggestions and comments. 
    \end{acknowledgements}

\end{document}